\documentclass[conference]{IEEEtran}
\IEEEoverridecommandlockouts

\usepackage{cite}
\usepackage{graphicx}
\usepackage{booktabs}
\usepackage{comment}
\usepackage{microtype}

\usepackage{amsmath}
\usepackage{amssymb}
\usepackage{amsfonts}
\usepackage{amsthm}
\usepackage{mathtools}
\usepackage{bbm}
\usepackage{multicol}

\usepackage{algorithm}
\usepackage{algorithmic}

\usepackage[hidelinks]{hyperref}
\usepackage[capitalize,noabbrev]{cleveref}

\newtheorem{lemma}{Lemma}

\DeclarePairedDelimiter{\norm}{\lVert}{\rVert}
\DeclarePairedDelimiter{\adjcurl}{\lbrace}{\rbrace}
\DeclarePairedDelimiter{\adjpar}{\lparen}{\rparen}
\DeclarePairedDelimiter{\abs}{\lvert}{\rvert}

\DeclareMathOperator*{\argmin}{arg\,min}
\DeclareMathOperator*{\argmax}{arg\,max}
\DeclareMathOperator{\conv}{conv}
\DeclareMathOperator{\eig}{eig}

\begin{document}

\title{Belief-Space Control for Personalized Cancer Treatment via Active Inference%
\thanks{The authors would like to acknowledge the American Association for Cancer Research and its material support in the development of the AACR Project GENIE registry, as well as members of the consortium for their commitment to data sharing. Interpretations are the responsibility of study authors.}
}

\author{\IEEEauthorblockN{Deniz Sargun \thanks{This work does not relate to the first two authors' positions at Amazon and HP.}}
\IEEEauthorblockA{
\textit{Amazon.com Inc.} \\
Palo Alto, CA \\
denizsargun@gmail.com}
\and
\IEEEauthorblockN{H. Bugra Tulay}
\IEEEauthorblockA{
\textit{HP Inc.}\\
Denver, CO \\
hbugratulay@gmail.com}
\and
\IEEEauthorblockN{C. Emre Koksal}
\IEEEauthorblockA{
\textit{The Ohio State University}\\
Columbus, OH \\
koksal.2@osu.edu}
}

\maketitle
\raggedbottom
\begin{abstract}
Cancer treatment is at the core a sequential decision-making problem with partial observability, latent patient heterogeneity, and explicit constraints on the budget for medical measurements. Unlike standard Reinforcement Learning (RL) approaches that control state trajectories, cancer treatments permanently modify patients' transition dynamics, changing how states evolve over time. We model cancer treatment as a belief-space planning problem using active inference, deriving an expected free-energy objective that unifies goal-directed control and information acquisition under measurement budgets without. We implement this framework using real clinical cancer data from the AACR Project GENIE Biopharma Collaborative dataset. Results on clinical data demonstrate a simultaneous patient categorization and high treatment efficacy, under real measurement and treatment constraints.

\end{abstract}


\section{Introduction}

Cancer treatment is a sequential decision-making problem under partial observability, uncertainty and explicit resource constraints. At each stage of therapy, clinicians select treatment actions (e.g., dosing, radiation) and, optionally, diagnostic measurements (e.g., imaging or laboratory tests) in order to control disease progression. The underlying patient state, including subjective circumstances such as the quality of life and tumor burden are only partially observed, and high-fidelity measurements are costly, invasive, and available at limited rates. Consequently, treatment decisions must be made under sparse, decision-dependent observations and a strict measurement budget. In addition, the response to treatment depends on latent patient-specific attributes (e.g. genetic, immunological, and physiological factors) that are not known a priori and must be inferred from heterogeneous clinical data. Therefore, effective decision-making requires a policy that simultaneously \emph{controls} disease progression and \emph{learns} patient-specific dynamics from limited feedback. This induces an exploration-exploitation tradeoff that is tightly coupled to measurement decisions: diagnostic actions improve state inference but consume scarce resources, while therapeutic actions affect both the state trajectory and the information available for future decisions.

From a modeling standpoint, cancer treatment differs from classical control problems. In standard optimal control or reinforcement learning formulations, actions influence the instantaneous state while the system dynamics remain mostly fixed. In contrast, oncologic interventions often induce \textbf{plastic dynamics}: treatments permanently alter disease mechanisms and shift long-term equilibria~\cite{nagy2005ecology}. As a result, the objective is not merely to control a state trajectory, but to shape the evolution of the \emph{distribution of the system state} over time. Hence, through the actions taken, the objective is to alter the system toward a desired preferred point, while simultaneously keeping the states at a desired distribution along the path, all subject to a measurement budget.

In this paper,\footnote{An extended version of this paper with additional proofs, derivations, and supplementary material is available on arXiv under the same title.} we first formalize general oncological treatment setting as a constrained partially observable Markov decision process (POMDP) in which actions drive the evolution of the belief state over latent patient variables. 
Treatment and measurement actions induce a controlled evolution of this belief through the predictive posterior, while performance is evaluated by how closely the resulting belief distribution aligns with a clinically preferred distribution. A constraint on cumulative or instantaneous measurements enforces a budget on information acquisition. Solving this POMDP exactly requires belief-state planning with unknown, high-dimensional transition and observation models under decision-dependent missingness, which is intractable with limited clinical data.

To address the complexity, we adopt the \textbf{free energy principle}~\cite{friston2019free,friston2023free} and its operational realization via \textbf{active inference}~\cite{ramstead2023bayesian}. Active inference recasts belief updating as variational inference and action selection as the minimization of expected free energy. To that end, it creates an information-theoretic functional that decomposes into terms corresponding to \textit{goal alignment (risk), observation uncertainty (ambiguity), and information gain (epistemic value)}. This formulation provides a tractable approximation to \textbf{belief-space planning} that naturally balances exploitation and exploration, without requiring explicit reward engineering or heuristic exploration strategies. It operates directly on belief distributions, making it well suited to distributional control under partial observability and measurement constraints.



We ground our modeling and evaluation using data from the AACR (American Association for Cancer Research) Project GENIE (Genomics Evidence Neoplasia Information Exchange) Biopharma Collaborative \cite{GENIE2017}, a large-scale clinicogenomic dataset that links genomic profiles with treatment histories, outcomes, and clinical annotations across diverse cancer types. GENIE captures key challenges inherent to oncology decision-making, including substantial patient heterogeneity, sparse and irregular measurements, and decision-dependent missingness, making it well suited for studying personalized treatment under limited feedback.

Rather than treating clinical and genomic variables as static covariates, our framework integrates GENIE data sequentially. Genomic and demographic features inform latent patient categories and prior beliefs, while observed treatments and clinical events update posterior beliefs over disease states via variational inference. Since high-fidelity measurements such as molecular profiling or advanced imaging are available only intermittently, our belief-based formulation naturally accommodates measurement constraints and uncertainty. By operating at the level of evolving state distributions rather than point-wise dynamics, the proposed approach aligns with the structure of oncology data and demonstrates \textbf{considerable increase in life expectancy} in clinical data driven simulations.

The contributions of this paper are as follows: 

\noindent $\bullet$ We formulate cancer treatment as a constrained distributional POMDP with plastic dynamics, emphasizing belief evolution rather than point-wise state control. 

\noindent $\bullet$ We develop an active inference framework that approximates belief-space planning via variational inference and expected free-energy minimization under measurement budgets. 

\noindent $\bullet$ We develop a data-driven instantiation of the expected free-energy framework that is explicitly aligned with the structure and limitations of real-world clinicogenomic data, enabling stable inference and decision-making under sparse, heterogeneous, and decision-dependent observations. 

\noindent $\bullet$ We demonstrate the practical efficacy of the proposed framework using real clinical data, showing explicit extension of life expectancy with our active-inference based personalized treatment strategy.

\section{Relevant Work}

Contextual bandits (CB) have become a dominant paradigm for personalized medicine due to their ability to map patient features (contexts) to treatments (actions) with strong regret guarantees \cite{li2010contextual}  \cite{pmlr-v32-agarwalb14}. However, this modeling choice is inherently myopic where rewards are immediate and tied to the current decision. Each treatment decision is evaluated based on immediate outcomes, with no formal mechanism to account for how current interventions reshape future state distributions or influence outcomes that manifest after substantial delays. Standard contextual bandit formulations abstract away action-dependent state dynamics, and therefore cannot represent how interventions reshape future state distributions. This is a structural mismatch for oncology and chronic disease management, where interventions change future state distributions. By ignoring the dynamics, CBs fail to account for the non-elastic nature of the costs or benefits of the current decisions.

RL in Markov decision processes (MDPs) explicitly models long-horizon consequences via controlled state transitions and discounted cumulative returns, making it a more appropriate abstraction for sequential treatment planning than contextual bandits. Deep RL has achieved major empirical successes in settings where extensive environment interaction is feasible, exemplified by value-based methods such as DQN. \cite{mnih2015human}. Recent retrospective deep-RL studies continue to report promising policies in critical care (e.g., sepsis), but they also sharpen concerns about offline learning and safety guaranties \cite{wu2023value}. In cancer treatment we need to incorporate additional constraints, such as (i) clinical rewards (e.g., tumor shrinkage) are sparse and delayed, (ii) observations are expensive and potentially invasive (labs, imaging, biopsies). Moreover, exploratory policies are not acceptable due to safety and ethical constraints. Previous work \cite{yu2021reinforcement} highlights these gaps and emphasizes partial observability, nonstationarity, and reliable off-policy evaluation \cite{jiang2016doubly} as central challenges for RL in healthcare.

When disease state is latent and only indirectly measured, the natural abstraction is a POMDP, where decisions are made over a belief state updated from noisy, intermittent observations. In many clinical settings, “when to observe” is itself a decision; observation-cost formulations make this explicit by attaching a cost to measurements and constraining adaptation between measurements \cite{reisinger2025markov}. Similarly, budget-constrained breast cancer screening policies and prostate cancer surveillance have been studied in \cite{helmeczi2023multi} and  \cite{LI2023386} respectively. However, maintaining accurate beliefs and planning in belief space is computationally demanding, motivating scalable approximate solvers such as point-based methods (PBVI \cite{pineau2003point}, SARSOP \cite{kurniawati2008sarsop}) and sampling-based online look-ahead with particle beliefs and scenario trees (POMCP \cite{silver2010monte}, DESPOT \cite{somani2013despot}). Some screening-policy work instead uses Markov models with Mixed-Integer Linear Programming (MILP) optimization to avoid direct MDP/POMDP solution at scale \cite{ccauglayan2025assessing}. Beyond POMDP solvers, \textbf{control-as-inference} and KL-regularized control formulations interpret decision-making as inference over trajectories and actions as structured modifications of transition dynamics, which closely aligns with transition-kernel ``tuning''~\cite{levine2018controlasinference,todorov2006linearly,kappen2012optimal}.

Despite substantial progress in scalable learning-based approximations in the POMDP space, optimizing for the belief space under limited observation budget remains underexplored. In oncology, observations are sparse, decision-dependent, and actions may affect the patient state and disease dynamics erratically due to latent categorical heterogeneity (e.g., patient subtypes). Indeed, clinical results on tumor progression shape mathematical models that highlight such heterogeneity that motivates latent categorical variables to capture unobserved variation in disease dynamics \cite{nagy2005ecology}. Moreover, evaluating the predictive posterior distribution of the disease evolution from Electronic Health Record (EHR) trajectories is prohibitively complex, even without irregular monitoring and limited measurement budgets. Thus, we move beyond exact Bayesian belief updates and optimal POMDP planning toward a control-as-inference framework via a systematic information-theoretic framework using free-energy principle. Using active inference, we jointly select measurements and interventions that balance progression toward desirable clinical outcomes and costly observations under latent categorical heterogeneity.
\section{Model and Problem}

We consider a partially observed controlled dynamical system indexed by \(C \in \mathcal{C}=\{1,\dots,d\}\). This variable will denote the latent \emph{patient category}, capturing coarse patient heterogeneity, representing directly observable attributes such as age or weight, as well as attributes that could be inferred via explicit tests and measurements, such as biological subtype or genetic markers. 

We assume the (possibly multivariate) latent \emph{patient state} \(X_k \in \mathcal{X}=\{1,\dots,n\}\) evolve according to a discrete Markov chain over treatment stages indexed by \(k = 0,1,\ldots\)

At each stage, the clinician may choose to acquire information through a binary (or discrete) \emph{measurement selector} \(M_k \in \{0,1\}\) (e.g., imaging vs.\ no imaging) and administer a \emph{treatment action} \(T_k\). Treatment actions modify the system's transition dynamics through: 
\begin{align} \label{equation: transition matrix update}
P_{k+1|C}& =(1-\alpha_{k|C})P_{k|C}+\alpha_{k|C}T_k
\end{align}
where \(\alpha_{k|C}\in[0,1]\) controls the tuning rate and the sequences \(\alpha=(\alpha_1,\dots,\alpha_d)\) are known category-specific parameters. This captures how oncologic interventions induce plastic dynamics: treatments permanently reshape disease progression rather than merely affecting the current state. We collect the decision variables into the composite action 
\begin{align}
A_k&\triangleq (M_{k+1},\,T_k) \in \mathcal{A}.
\end{align}
for the available set, ${\cal A}$ of actions. State observations are intermittent and controlled by an observation variable as $Y_k= M_k X_k$, $M_k$ being the indicator for the direct (or noisy) measurement of state $X_k$.

\begin{figure}[ht]
\vskip -0.2in
\begin{center}
\centerline{\includegraphics[width=0.38\linewidth]{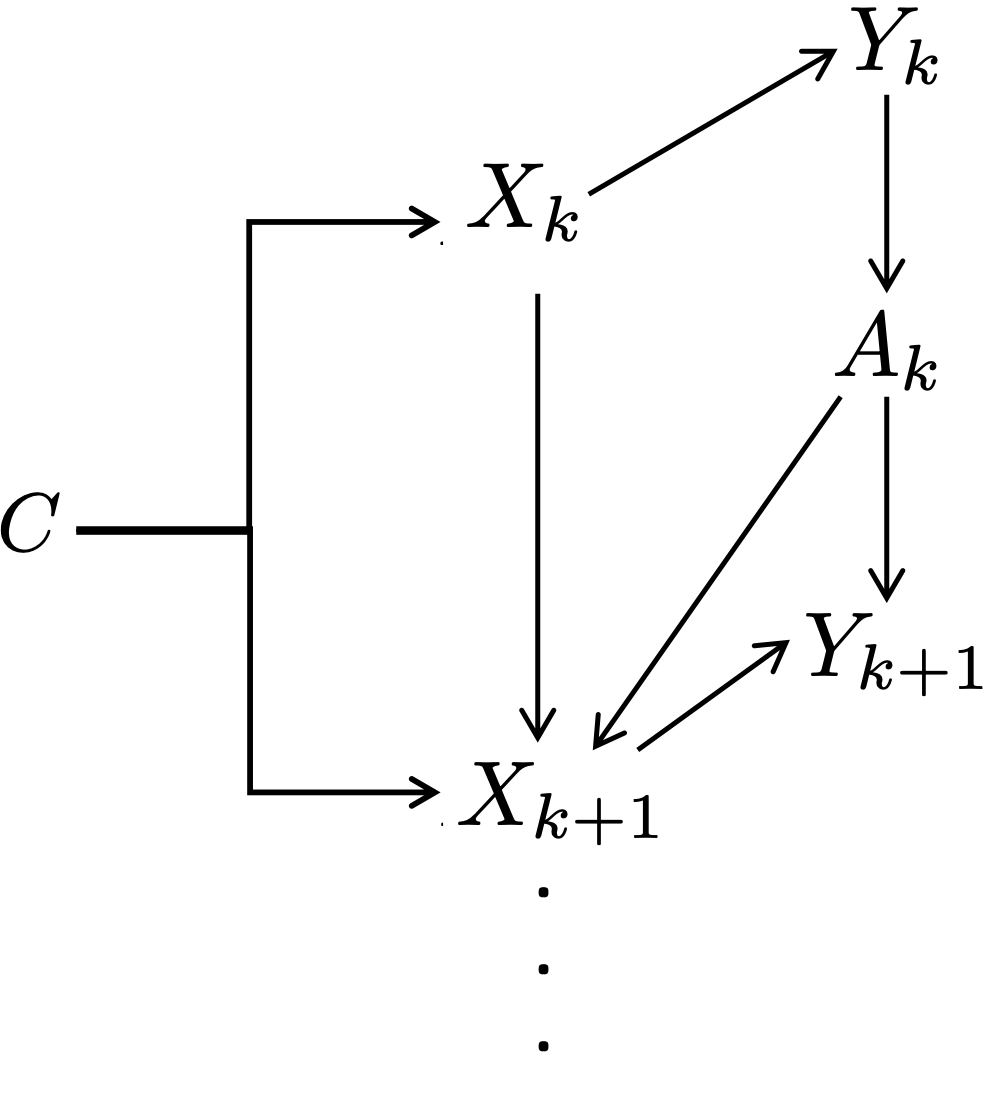}}
\caption{Patient dynamics with limited measurements. Patient category \(C\) and state \(X_k\) evolve under actions $A_k$; observations \(Y_k\) are obtained via a measurement decision (e.g., imaging).}
\vspace{-0.2in}
\label{figure: model}
\end{center}
\end{figure}
The system is initialized with category-specific state distributions \(p(X_0|C)\) and baseline transition matrices \(P_{0|C}\) extracted from clinical data. The evolution of the system is illustrated in Fig. \ref{figure: model}. The associated generative factorization can be summarized as: \vspace{-0.1in}
\begin{equation}
p_C(x_{k+1}, y_k, a_k | x_k)
= p_C(x_{k+1} | x_k, a_k)p(a_k | y_k)p(y_k | x_k)
\nonumber
\end{equation}
where \(p_C(x_{k+1} | x_k, a_k)\) captures category-dependent disease/treatment dynamics, \(p(y_k | x_k)\) is the EHR observation channel, and \(p(a_k | y_k)\) represents an implicit decision rule.

\noindent \textbf{Problem formulation:} A distinctive aspect of our setting is that the objective is naturally expressed at the \emph{distributional} level: actions do not merely move a single realized trajectory, but shape the evolution of the \emph{belief} (or state distribution) over the patient under limited feedback. Accordingly, we start by formulating the control problem in belief space. Let 
\( \pi_{k,X,C}(x,c)\triangleq p(x,c \mid y_{0:k},a_{0:k-1}) \)
denote the posterior belief over the latent patient state and category after observing history up to stage \(k\). The belief evolves via the Bayes filter: 
\begin{equation}
\pi_{k+1,X,C}(x',c)\propto p(y_{k+1}| x')\sum_{x\in\mathcal{X}}p_c(x'| x,a_k)\pi_{k,X,C}(x,c).
\label{equation: bayes_filter}
\end{equation}
We seek actions that drive \(\pi_{X}\) toward a clinically desired distribution \(\pi_{X}^*\) (equivalently, toward preferred state marginals), rather than optimizing a pointwise state cost. Over an $n$-stage treatment, our objective can be stated as a constrained distributional POMDP: 
\begin{align}
\min_{p_A} &\quad \mathbb{E}_{p_A}\left[\sum_{k=0}^{n} KL\big(\pi_{X,k} \,\|\, \pi_{X}^*\big)\right] \label{equation: pomdp_dist_obj}\\
\text{subject to}&\quad  \mathbb{E}_{p_A}\!\left[\sum_{k=0}^{n} M_k\right]\leq B,
\label{equation: pomdp_dist_constraint}
\end{align}
where \(KL\) is the Kullback-Leibler divergence, measuring mismatch between the current belief and the preferred belief (we will also use Jensen-Shannon divergence later on), and \(B\) is a budget on cumulative measurements. This formulation captures the measurement actions' influence on the \emph{evolution of the belief distribution} through \eqref{equation: bayes_filter}, and that control seeks to shape the entire distribution (personalized uncertainty included), not only the realized state.

Despite its clinical relevance, solving \eqref{equation: pomdp_dist_obj}--\eqref{equation: pomdp_dist_constraint} is difficult because it requires accurate belief-state dynamics, which in turn depend on the predictive posterior induced by the unknown transition kernel \(p_c(x_{k+1}| x_k,a_k)\) and the observation model \(p(y_k| x_k)\), under decision-dependent missingness, driven by \(M_k\). In the following, we introduce a free-energy formulation that bypasses explicit belief-space planning by replacing exact Bayesian filtering with variational inference and deriving an expected free-energy objective whose risk term naturally corresponds to driving \(\pi_{k,X}\) toward \(\pi_{X}^*\), while its epistemic component quantifies the value of additional measurements under the budget constraint.

\section{Approach} \label{section: approach}

Directly learning or exploiting the predictive posterior
\(
p_C(x_{k+1} | x_k, y_k, a_k)
\)
is challenging in clinical data due to sparse measurements, 
and observations that depend on past decisions, and latent heterogeneity through \(C\).
The \textbf{free-energy principle}~\cite{friston2019free,friston2023free,ramstead2023bayesian} provides a unified alternative in which inference and
control are treated jointly.
Rather than computing exact Bayesian updates and optimal policies in a POMDP,
we maintain an approximate belief over latent variables and select actions that
are expected to minimize a principled information-theoretic objective.

Furthermore, our formulation will enable an explicit tradeoff between (i) driving the system
toward clinically preferred states (e.g., tumor reduction) and (ii) acquiring
information efficiently under measurement constraints, under limited budgets.

\subsection{Variational Free Energy} \label{section: vfe}

Consider a candidate action \(a \in \mathcal{A}\).
The joint generative model for the next step is 
\begin{multline}
p(x_{k+1}, y_{k+1},c \mid x_k, y_k, a)
\\
= p_C(x_{k+1} | x_k, a)\,
  p(y_{k+1} | x_{k+1})\,
  p(c). 
\label{equation: joint_next}
\end{multline}
Exact inference of \(p(x_{k+1},c \mid y_{k+1}, x_k, y_k, a)\) is generally intractable, so we introduce a variational density \(\pi_{k+1,X,C}(x_{k+1},c \mid y_{k+1}, x_k, y_k, a)\).
The corresponding \emph{variational free energy}~\cite{friston2019free} evaluated at time $k+1$ is 
\begin{multline}
\mathcal{F}_{k+1}(y_{k+1}, a)\\
\triangleq\mathbb{E}_{\pi_{k+1,X,C}}
\Big[\log \pi_{k+1,X,C}(x_{k+1},c \mid y_{k+1}, x_k, y_k, a)\\
-\log p(x_{k+1}, y_{k+1},c \mid x_k, y_k, a)\Big].
\label{equation: vfe_def}
\end{multline} 
Using the standard decomposition, 
\begin{align}
\nonumber
&\mathcal{F}_{k+1}(y_{k+1}, a)
=KL\left(\pi_{k+1,X,C}(x_{k+1},c  \mid y_{k+1}, x_k, y_k, a) \right. \\
& \|  
\left. p(x_{k+1},c \mid y_{k+1}, x_k, y_k, a)
\right) 
- \log p(y_{k+1} | x_k, y_k, a),
\label{equation: vfe_bound}
\end{align}
minimizing \(\mathcal{F}_{k+1}\) with respect to \(\pi_{k+1,X,C}\) yields to approximate Bayesian inference and an upper bound on surprise~\cite{friston2017computational}.

\subsection{Expected Free Energy and Action Selection} \label{section: efe}

Action selection occurs before observing \(Y_{k+1}\).
Given a belief \(\pi_{k,X,C}(x_k, c \mid y_{0:k}, a_{0:k-1})\),
a candidate action \(a\) induces the predictive distribution \vspace{-0.1in}
\begin{multline}
\pi_{k+1,X,Y,C}(x_{k+1},y_{k+1},c)
\\
\triangleq
\Big(
\mathbb{E}_{\pi_{k,X,C}(x_k, C)}\left[
p_C(x_{k+1} | x_k, a) \right]
\Big)\,
p(y_{k+1} | x_{k+1}).
\label{equation: predictive_q}
\end{multline}
The \emph{expected free energy} is defined as the expected future variational free energy: \vspace{-0.1in}
\begin{multline}
\mathcal{G}_k(a)
\triangleq
\mathbb{E}_{\pi_{k+1,y}(y_{k+1})}
\big[
\mathcal{F}_{k+1}(y_{k+1}, a)
\big]
\\
=
\mathbb{E}_{\pi_{}(y_{k+1}, z_{k+1})}
\Big[
\log \pi_{k,X,C}(z_{k+1} | y_{k+1}, a)
\\ -
\log p(z_{k+1}, y_{k+1} | a)
\Big].
\label{equation: efe_def}
\end{multline}
To encode clinical objectives, we introduce a preference distribution
\(\pi_{X}^*(x_{k+1})\) over desired next-step states.
Via the factorization\footnote{As a part of the agent’s generative model, \(\pi_X^*(x_{k+1})\) encodes desired outcomes, representing a preference prior rather than the true disease dynamics. Actions are chosen to shape the predictive belief toward this distribution.} \(p(z_{k+1}, y_{k+1} | a)= p(y_{k+1} | x_{k+1}, a)\pi_X^*(x_{k+1})p(c),\)
the expected free energy decomposes as 
\begin{multline}
\mathcal{G}_k(a)
=KL\big(
\pi_{k+1,X}(x_{k+1})\|\pi_X^*(x_{k+1})
\big)+\\
\mathbb{E}_{\pi_{k+1,X}(x_{k+1})}\left[H\left(p(y_{k+1} | x_{k+1},a)\right)\right]-I_\pi(X_{k+1}; Y_{k+1} | a).
\label{equation: efe_decomp}
\end{multline}
The decomposition in \eqref{equation: efe_decomp} clarifies how action selection balances clinical objectives with information acquisition. 
The \textbf{risk term}, \(KL\big(\pi_{k+1,X}(x_{k+1})\|\pi_X^*(x_{k+1})\big),\) penalizes predicted state distributions that deviate from clinically preferred outcomes, promoting exploitative actions that directly reduce tumor burden or adverse events. 
The \textbf{ambiguity term}, \(\mathbb{E}_{\pi_{k+1,X}(x_{k+1})}[H(p(y_{k+1}|x_{k+1},a))],\)
discourages actions expected to produce intrinsically noisy or uninformative observations, reflecting measurement cost and clinical burden. 
The \textbf{epistemic term}, \(I_\pi(X_{k+1};Y_{k+1}|a),\) quantifies expected information gain, favoring exploratory actions such as imaging when uncertainty about patient state or category is high.

At each point $k$ in time, the optimal decision rule is therefore 
\begin{align}
a_k^\star&\in \arg\min_{a \in \mathcal{A}} \mathcal{G}_k(a),
\end{align}
which balances goal-directed behavior (risk term) with information acquisition (epistemic value) under limited measurements. Measurement selector \(M_{k+1}\) directly modulates the information term \(I_\pi(X_{k+1}; Y_{k+1} | a)\), making the measurement budget an intrinsic component of the control objective.

\subsection{Simplified Cost Function} \label{subsection: simplified_obj}

Here, we modify the cost function provided by expected free energy in order to align it with the medical data we will use. The alternate cost preserves the core structure of expected free-energy minimization while allowing easier integration of EHR data. In particular, we consider 
\begin{align} \label{equation: simplified_cost}
J_k&\triangleq M_k+\tau L(\pi_{k,X}, X_k)+\kappa\mathrm{JSD}\left(\pi_{k,X} \|\pi_X^*\right),
\end{align}
where \(\pi_{k,X}\) denotes the state belief distribution at stage \(k\), \(M_k\) is the measurement action (e.g., imaging), \(L(\cdot,\cdot)\) is a task-specific distance capturing ambiguity with respect to the current belief over the patient state \(X_k\), and \(\pi_X^*\) is a preferred action distribution encoding clinically desired behavior. 
We employ the Jensen--Shannon divergence (JSD) as a symmetric and bounded measure of discrepancy between distributions. For two distributions \(p\) and \(q\), it is defined as 
\begin{align} \label{equation: jsd_def}
\mathrm{JSD}(p \,\|\, q)&\triangleq
\frac{1}{2}KL\!\left(p \,\middle\|\, \tfrac{1}{2}(p+q)\right)
+
\frac{1}{2}KL\left(q \,\middle\| \tfrac{1}{2}(p+q)\right).
\end{align}
Unlike KL divergence, JSD is always finite and well defined, even when the supports of the two distributions do not fully overlap, making it well suited for robust optimization with limited medical datasets. We use $\tau$ and $\kappa$ to calibrate for the desired balance between different terms of the cost function.

This objective mirrors the canonical decomposition of the expected free energy in \eqref{equation: efe_decomp}. The measurement term \(M_k\) represents an explicit \emph{epistemic cost}, reflecting the resource burden associated with information acquisition. The term \(L(\pi_{k,X},X_k)\) serves as a surrogate for the \emph{ambiguity} component of expected free energy, penalizing actions that are poorly aligned with the current belief over patient state or that are expected to produce unreliable or low-utility observations. The final term,
\(\mathrm{JSD}(\pi_{k,X}\|\pi_X^*)\), acts as a \emph{risk} term, incentivizing the action distribution to remain close to clinically preferred treatments, thereby promoting exploitation.

Given the objective in Eq. (\ref{equation: simplified_cost}), the problem reduces to: 
\begin{equation} \label{equation: minimize total cost}
    \min_{A_k, k\geq 1}\mathbb{E} \left[\sum_{k=0}^{\infty} \gamma^k J_k \right]
\end{equation}
where \(\gamma \in (0,1)\) is the discount factor.
Overall, \eqref{equation: minimize total cost} can be viewed as a structured surrogate for expected free-energy minimization: it retains the essential tradeoff between epistemic exploration, ambiguity reduction, and risk-sensitive exploitation, while replacing implicit information-theoretic quantities with explicit, data-aligned penalties. This modification is introduced to enable easier integration of
EHR data, as illustrated in Sections~\ref{sec:data_integration} and \ref{sec:results}.
\section{Solution and Basic Performance Limits}
\label{section: solution}
In this section, we propose and analyze two naive solutions and then compare the two methods with a genie algorithm. We denote the policy on the decision of monitoring as \(\mu\) and the policy on the tuning of the state transition matrix as \(\theta\).

\subsection{Solution Agents} \label{subsection: solution agents}
\noindent \textbf{Constant tuning agent:} For the naive solution we use a constant matrix for tuning and observe only when entropy of belief exceeds a threshold. 
\begin{align}
    \mu_k&= \mathbbm{1}(H(\pi_{k,X})\geq h), \theta_k= T.
\end{align}
Proper choice of the value of \(T\) is addressed in Section \ref{subsection: finding stable transition matrices}.

\noindent \textbf{Convex tuning agent:}
 Unlike the constant tuning agent, the convex tuning agent uses different constant matrices to tune each category and a convex combination for the current belief. 
\begin{align}
    \mu_k&= \mathbbm{1}(H(\pi_{X,k})\geq h), \theta_k=\sum_c\pi_{k,C}(c)T_c.
\end{align}
Proper choice of the value of \(T_c\) is addressed in Section \ref{subsection: finding stable transition matrices}.

\noindent \textbf{Constant discrete tuning agent:}
With tuning decisions restricted to a finite set \(\mathcal{T}\) state and decision space for the agent changes: \hspace{-0.15in}
\begin{equation}
    S= \{\pi,(P_k)_{c=1}^d,k,(\alpha_k^\infty)_{c=1}^d\},\ 
    A= \{M=0,1\}\times\mathcal{T}
\end{equation}
For each state, we have a value function \(V:S\to\mathbb{R}\) but even under time homogeneity, i.e., \(\alpha\) the characteristics are time invariant, the state space is large and \(S=\{\pi,(P_k)_{c=1}^d\}\).

\noindent \textbf{Discrete tuning genie:}
On the other hand, if the category is known, such as the case of a genie algorithm, then we can reduce the state space to a tractable dimension \(S=\pi,P_{k,C=c}\).

\subsection{Finding Stable Transition Matrices} \label{subsection: finding stable transition matrices}
Given the desired state \(\pi^*\) we have multiple heuristics to determine transition matrices \(P^*\) with desirable properties. We have the following constraints: 
\begin{equation} \label{equation: stable}
    (P^*-I)\pi^*\sim 0,\ \ \ P^*e_n=e_n
\end{equation}
where the second equation is only if state \(n\) is the unique absorbing state. Since \(P^*\in[0,1]^{n\times n}\), we have an underdetermined linear system of \(n^2\) variables and \(2n\) equations.

\noindent \textbf{Metropolis-Hastings construction:}
Choose any symmetric proposal matrix \(P\). Then, 
\begin{align}
    \hspace{-0.1in} 
    P^*_{MH}(i,j)&= \begin{cases}
     \min(1, \pi^*(j)/\pi^*(i))P(i,j), & i\neq j \\
     1 - \sum_{j\neq i} P^*_{MH}(i,j), & i=j
\end{cases}
\end{align}
Finally, update transition out of absorbing state as \(0\) if it exists.

\noindent \textbf{Maximum second largest absolute-value solution:}
If there are no absorbing states, choose a transition matrix that is ergodic and has maximum second largest absolute value eigenvalue \(\abs{\lambda_2(P)}\):
\begin{equation}
    P^*_{SL}= \argmax_{\pi^* P=\pi^*, P\ ergodic}\abs{\lambda_2(P)}
\end{equation}

\noindent \textbf{Minimum distance solution:}
Choose an ergodic transition matrix that satisfies the steady-state condition and is closest to the mixture of transition matrices:
\begin{equation}
    P^*_{MD}= \argmin_{\pi^* P=\pi^*, P\ ergodic}\norm{P-\sum_c p_C(c)P_{0|c}}
\end{equation}

\subsection{Belief Update}
Given belief \(\pi_k\) we update it as follows. If \(M_{k+1}=1\) then \(Y_{k+1}=X_{k+1}\) and
\begin{align}
    \pi_{k+1,X,C}(x,c)&= \tfrac{\mathbbm{1}_{X_{k+1}}(x)\sum_{x_k}\pi_{k,X,c}(x_k,c)e_xP_{k+1|c}e_{x_k}^T}{\sum_{x_k,c'}\pi_{k,X,c}(x_k,c')e_xP_{k+1|c'}e_{x_k}^T}.
\end{align}
where \(e_1,\dots,e_n\) denote the standard basis for \(\mathbb{R}^{1\times n}\). Otherwise, \(M_{k+1}=0, Y_{k+1}=\emptyset\) and
\begin{align}
    \pi_{k+1,X,C}(x,c)&= \sum_{x_k} \pi_{k,X,C}(x_k,c) \cdot e_x P_{k+1|c} e_{x_k}^T
\end{align}
Optimal sequence of policies \((\mu^*,\theta^*)\) is composed of optimal instantaneous policies \((\mu_k^*, \theta_k^*)\) indexed in time \(k\) that are functions of the belief over states at time \(k\) (\(\pi_k\)), the belief over categories \(b\) and \(\alpha\).

\subsection{Life Expectancy}
In this section, we derive the remaining life expectancy of a patient at a given state. Let $v$ denote the $n$-dimensional vector with entry $v_i$ representing the expected remaining life of a patient in state $i$. The time is measured in number of transitions, so it has the same unit as the time duration of a slot.

We characterize the problem as one of Markov chains with rewards. Since one of the states is terminal (deceased), the steady-state probability of that state is $1$. As a result, $\pi = [0\ 0 \ldots 0\ 1]^{\mathsf T}$, where state $n$ represents the terminal state. We assign a reward of $1$ unit for visiting each other state, and as a result the associated reward vector is $r=[0\ 0 \ldots 0\ 1]^{\mathsf T}$. Then, we can write the following equation for $v$:
\begin{equation}
    \label{eq:life_expectancy}
    v=r+P v.
\end{equation}
Here, the mechanism is clear: life expectancy is composed of the immediate reward plus the remaining life expectancy after making the transition. The following lemma shows that there is a unique solution for Eq.~(\ref{eq:life_expectancy}):
\begin{lemma} \label{lemma: life expectancy}
The equation $v=r+P v$ has a unique solution.
\end{lemma}
\begin{proof}
See Appendix \ref{section: proof of life expectancy lemma}.
\end{proof}

\subsection{Basic Limits}
Due to space constraints, we give the achievable fundamental limits of our solutions as a single unified lemma, instead of a sequence of theorems. 

\begin{lemma} \label{lemma: eigenvalue of convex combination}
Given category \(C\) and an aperiodic state transition matrix \(P_{k+1|C}\) with a single absorbing state and \(n\) distinct eigenvalues \(\lambda_1^n\) such that \(\abs{\lambda_i}\geq\abs{\lambda_j}\) for all \(1\leq i<j\leq n\),

\begin{enumerate}
    \item the set of its eigenvalues \(\text{eig}(P_{k+1|C})\) is a subset of the convex hull of the eigenvalues of \(P_{k|C}\) and \(T_k\),
    \item for any \(\alpha_{k|C}\) and \(\lambda\in\eig P_{k+1|C}\), there exists a \(\lambda'\in\eig P_{k|C}\) such that, \(\abs{\lambda}\leq \abs{\lambda'}+\alpha_{k|C}2\sqrt{2}\),
   
    \item for the second largest absolute value eigenvalues \(\lambda_2,\lambda_2',\lambda_2''\) of \(P_{k+1|C},P_{k|C},T_k\), we have \(\abs{\lambda_2}\leq \max\{\abs{\lambda_2'},\abs{\lambda_2''}\}\),
    \item for a constant \(P_{k+1|C}=P\), \(P^k=\sum_{i=1}^n\lambda_i^ku_iv_i\) where \(u_i\in\mathbb{R}^{n\times 1}\) and \(v_i\in\mathbb{R}^{1\times n}\) are right and left eigenvectors corresponding to \(\lambda_i\), \(\lambda_1=1\) and for all \(i>1\), \(\abs{\lambda_i}<1\), \(v_1=e_n\) is the unique steady state distribution, \(\norm{P^k-P_\infty}_{\text{max}}\to 0\) exponentially where maximum norm \(\norm{\cdot}_{\text{max}}\) is the element-wise norm and for any initial distribution \(\pi_{0,X}\) and distribution \(g\), \(\pi_{0,X}P^k\to v_1, H(\pi_{0,X}P^k)\to H(v_1)=0\) and \(JSD(\pi_{0,X}P^k\|g)\to JSD(v_1\|g)\) exponentially.
\end{enumerate}
\end{lemma}
\begin{proof}
See Appendix \ref{section: proof of lemma}.
\end{proof}
\section{Data Integration}
\label{sec:data_integration}

We use the AACR Project GENIE Biopharma Collaborative (BPC) colorectal cancer (CRC) v2.0-public data release as a real-world longitudinal clinical dataset to support algorithm development in this phase. The public CRC release includes 1,485 patients treated at Memorial Sloan Kettering Cancer Center (MSKCC), Dana-Farber Cancer Institute (DFCI), and Vanderbilt-Ingram Cancer Center (VICC), and provides harmonized patient-level clinical trajectories across diagnosis, treatment, and outcomes.

To construct patient trajectories, we infer discrete disease states
$\mathcal{S}=\{A,B,C,D\}$
from longitudinal monitoring signals. Imaging and medical oncologist assessments provide curated categorical status updates (improving/responding, stable/no change, progressing/worsening, or not stated/intermediate). We map these longitudinal observations to four states: \textbf{Attenuation of tumor} ($A$), \textbf{Balanced response} ($B$), \textbf{Critical condition} ($C$), and the terminal absorbing state \textbf{Deceased} ($D$). The Medical Oncologist Assessment is curated at around one assessment per month, enabling a standardized longitudinal view of progression aligned with imaging events (Fig.~\ref{fig:patient_timeline}).

We estimate baseline (treatment-free) transition dynamics by extracting transitions from time windows without active treatment to approximate natural disease progression, while death transitions are retained across treated and untreated intervals to better capture realistic mortality risk (Fig.~\ref{fig:transitions}). Summary statistics for observation events and transition counts are reported in Table~\ref{tab:genie_crc_processing} in the appendix.
\begin{figure}[!t]
  \centering
  \includegraphics[width=0.85\columnwidth]{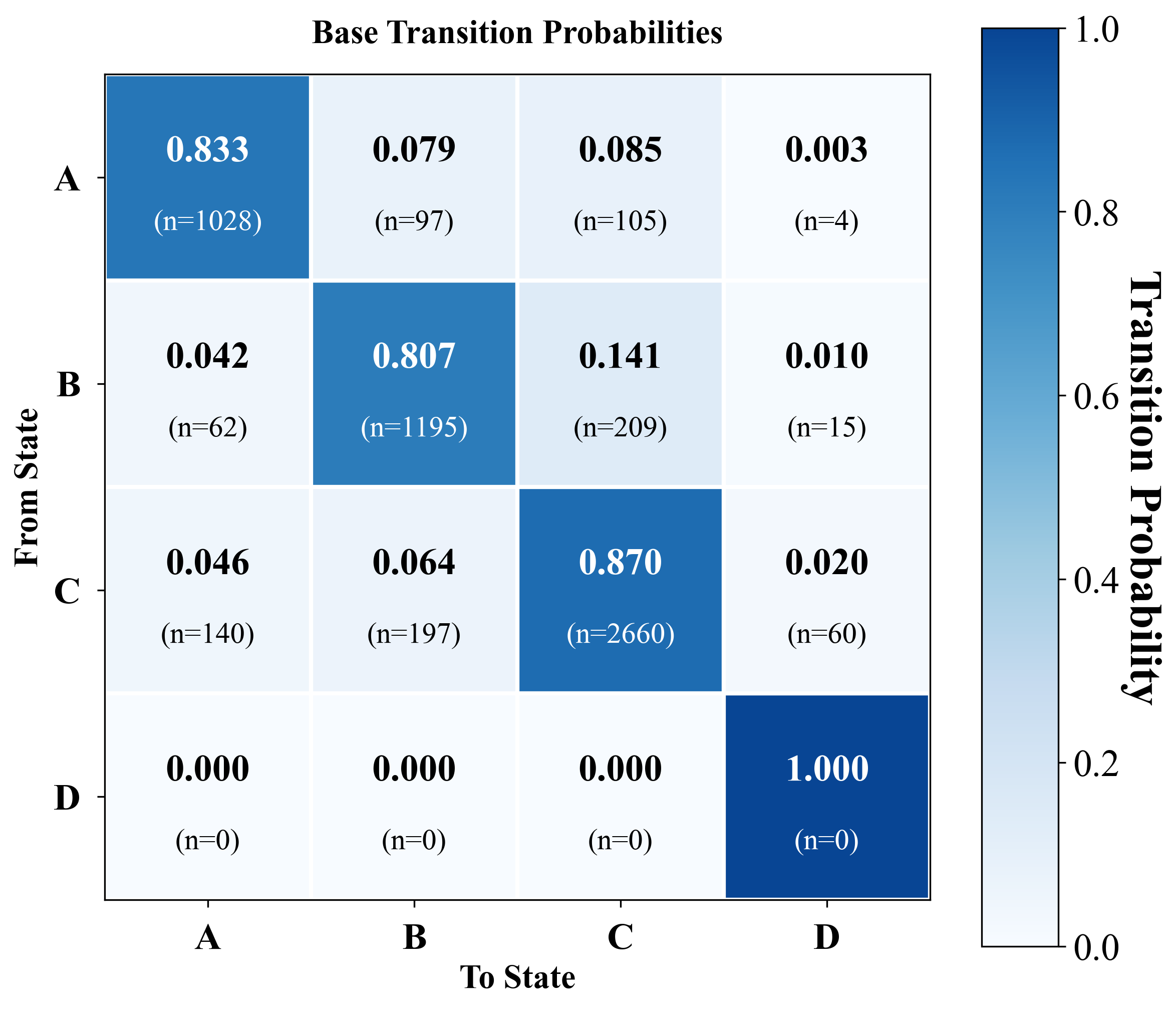}
  \includegraphics[width=0.85\columnwidth]{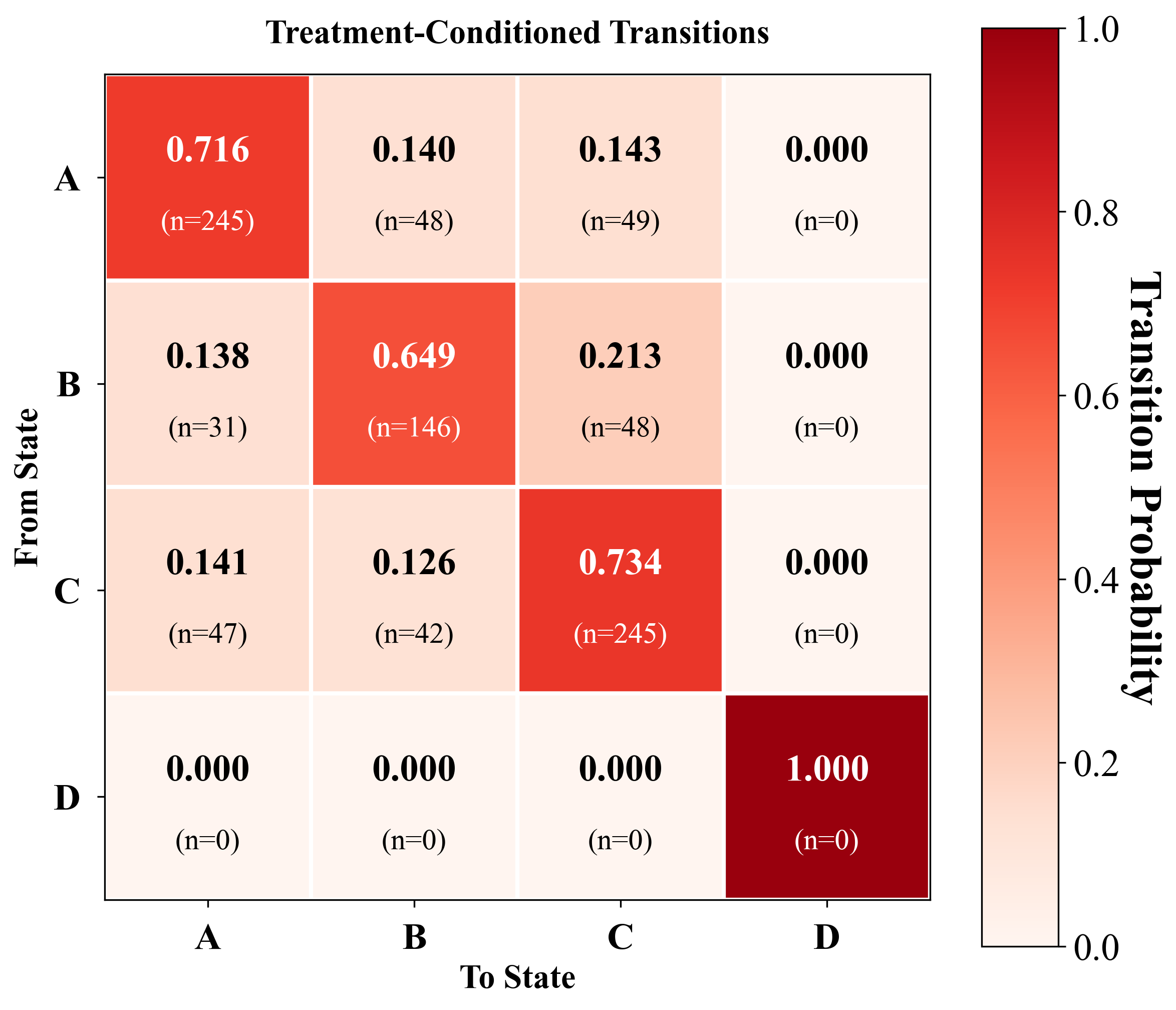}
  \caption{Empirical transition probabilities. \textbf{Top}: Baseline (untreated) transitions $P_{0\mid C}$ for early-stage, middle-aged patients. \textbf{Bottom}: Treatment-conditioned transitions during chemo.}
  \label{fig:transitions}
  \vspace{-0.1in}
\end{figure}
For treatment-conditioned dynamics, we parse regimen records into contiguous treatment windows and map each window to one of seven action classes (EGFR, VEGF, BRAF, HER2, IO, ChemoOnly, Investigational) using the following priority: check targeted agents first, then IO, and default to ChemoOnly when only cytotoxic agents are present (see Table~\ref{tab:genie_crc_processing} in the appendix). We estimate empirical action-conditioned transition matrices $P_{\cdot \mid c, a}$ for each category $c$ and action $a$ from the observed transitions that occur while $a$ is active when the observation begins. In CRC, VEGF/ChemoOnly dominate the data support, while IO/HER2/BRAF are sparse and should be interpreted with caution.

Finally, to capture clinically meaningful heterogeneity without using genomics, we stratify patients by stage at diagnosis (early stage I--III vs. advanced stage IV) and by age group ($<50$, 50--70, $>70$) (see Table~\ref{tab:genie_crc_processing} in the appendix).
\section{Results}
\label{sec:results}

We simulate a patient trajectory with the Constant Tuning Agent in Section \ref{subsection: solution agents} using an entropy threshold \(h=1.6\) over a time horizon of \(K = 100\) steps, representing approximately \(2-3\) years of treatment. Tuning limits \(\alpha_k\) are drawn from an exponential distribution with a time-decaying mean. Figure \ref{figure: true and belief state} shows the true patient state oscillating primarily between Attenuation (A) and Balanced (B) states over 100 time steps, with occasional transitions to Critical (C). The agent's belief state distribution tracks the transitions through marginalized probabilities and successfully captures the underlying state dynamics despite partial observability.
\begin{figure}[!t]
\centering
\includegraphics[width=0.85\columnwidth]{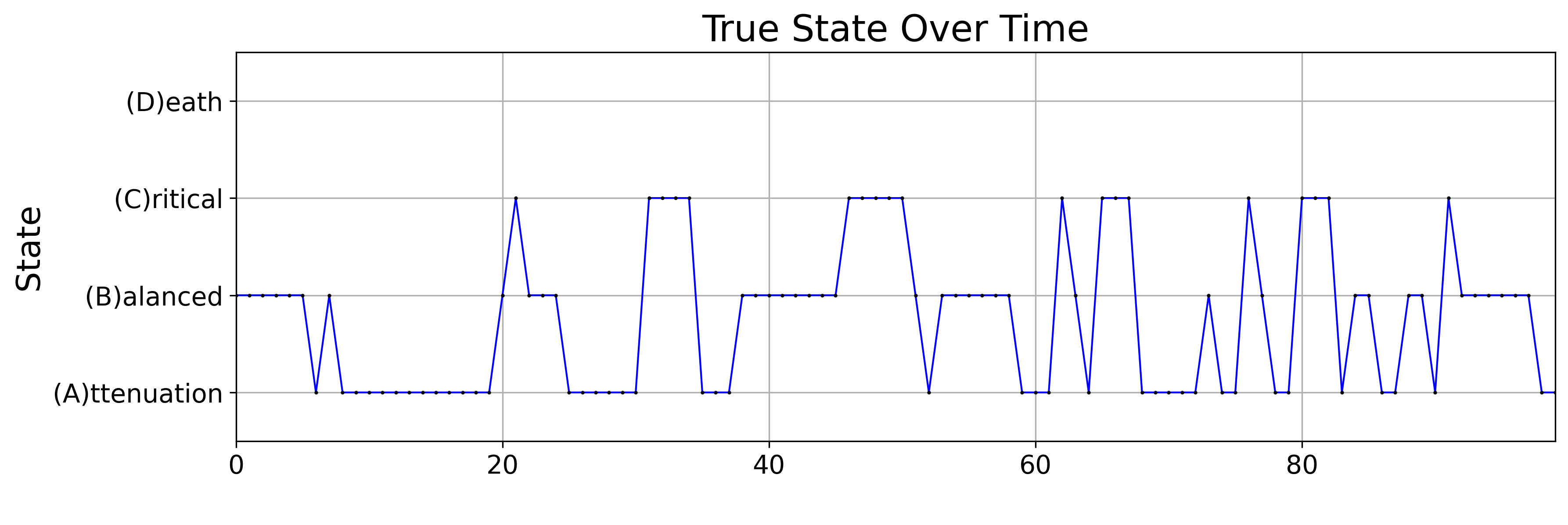}
\includegraphics[width=0.85\columnwidth]{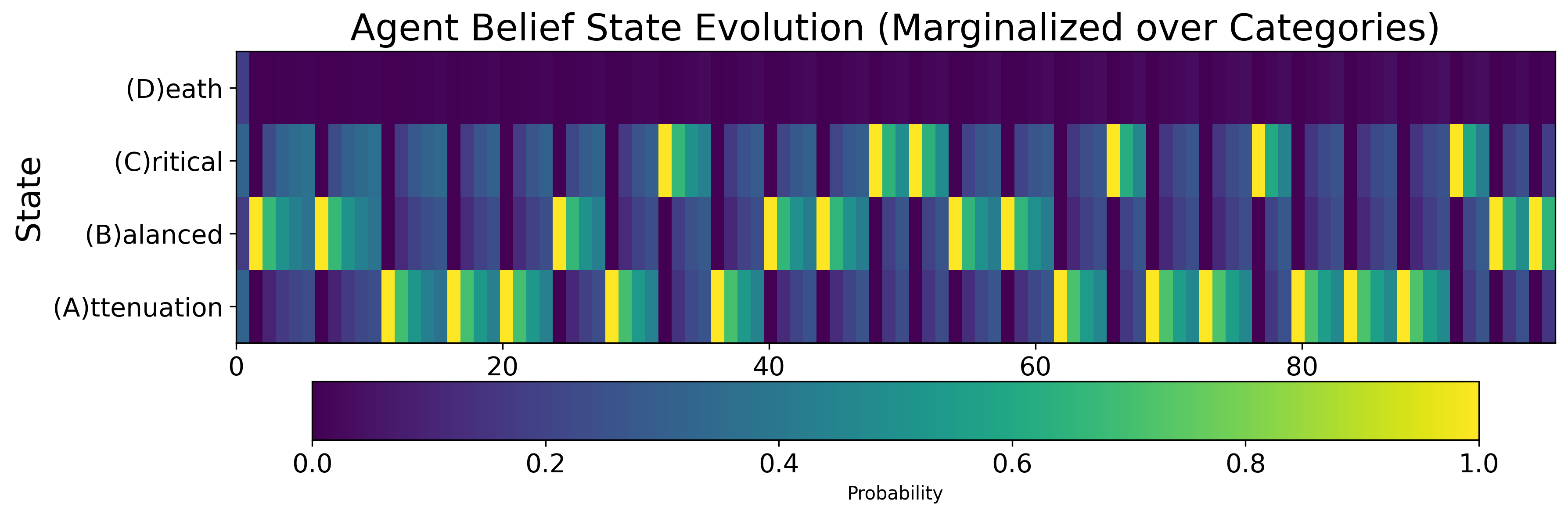}
\vspace{-0.2in}
\caption{True patient state (top) and agent belief distribution (bottom) over 100 time steps. The belief successfully tracks the underlying state dynamics despite partial observability.}
\label{figure: true and belief state}
\vspace{-0.1in}
\end{figure}
The entropy threshold measurement policy triggers observations approximately every 4--5 steps when belief entropy exceeds \(h=1.6\). This creates a characteristic sawtooth pattern where entropy accumulates between measurements and resets upon observation, demonstrating the agent's adaptive information acquisition strategy under budget constraints (Figure~\ref{figure: measurement decisions}). The regular measurement intervals balance state tracking accuracy with resource conservation. This pattern indicates that measurements are concentrated at higher-uncertainty periods and are suppressed when the belief is more confident. \vspace{-0.1in}
\begin{figure}[ht]
  \begin{center}
    \includegraphics[width=0.85\columnwidth]{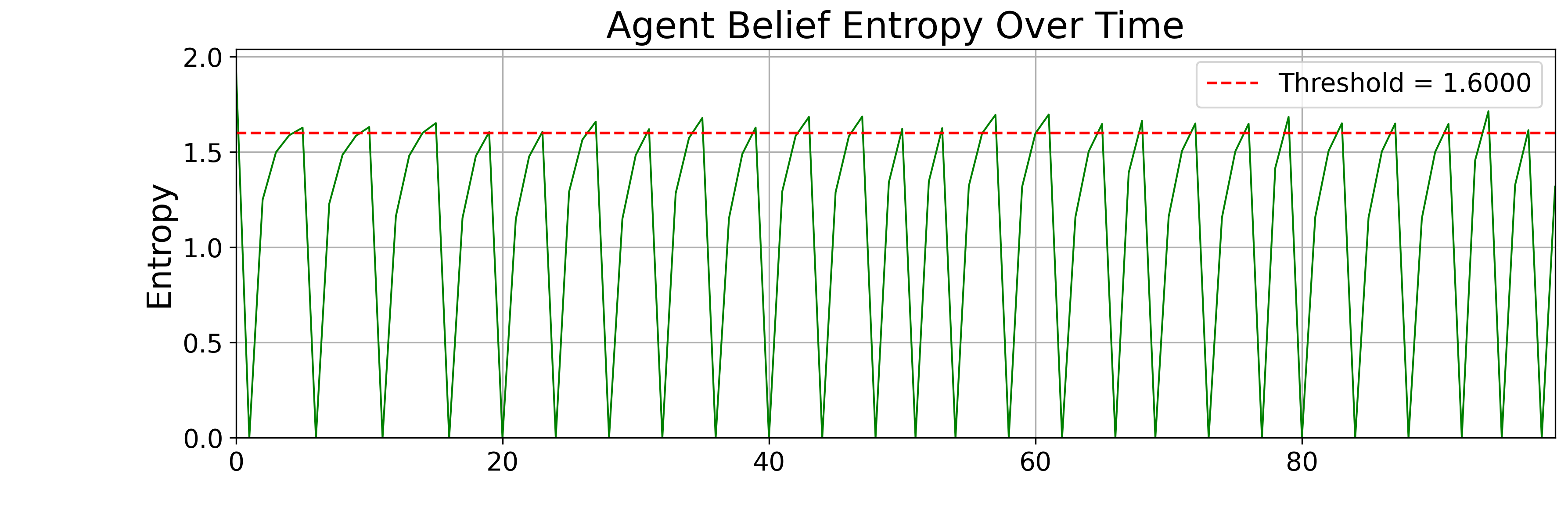}
    \vspace{-0.25in}
    \caption{Measurement decisions and belief entropy}
    \label{figure: measurement decisions}
  \end{center}
\vspace{-0.15in}
\end{figure}
A key finding emerges in the category identification dynamics (Figure \ref{figure: category belief evolution}): the true patient category (Category 5, brown line) is detected after 100 time steps, with Category 1 also receiving high posterior probability. This reveals observational equivalence between categories and demonstrates that the entropy-threshold measurement policy prioritizes state tracking as well as active category identification. Multiple categories produce similar observable state trajectories, making discrimination challenging with limited data.
\begin{figure}[ht]
  \begin{center}
\centerline{\includegraphics[width=0.85\columnwidth]{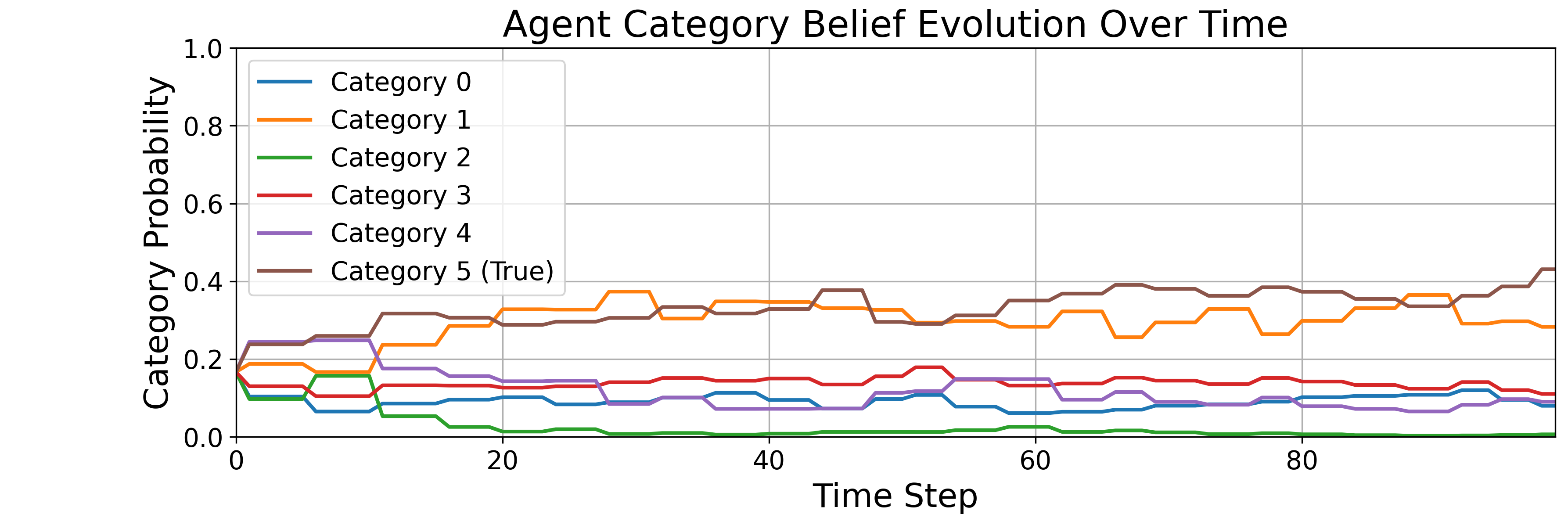}}
    \caption{Category belief evolution}
    \label{figure: category belief evolution}
  \end{center}
\vspace{-0.2in}
\end{figure}
Alongside correct category identification, the agent also maintains robust control performance. Figure \ref{figure: L2 distance evolution} shows decreasing tuning matrix error ($\ell_2$ distance down from \(\sim\) 1.2) between the constant tuning matrix and true category-specific transition matrix throughout the simulation. This persistent mismatch arises since the tuning mechanism only partially shifts the transition dynamics toward the target matrix $T$ at each step, with the tuning rate \(\alpha_k\) controlling the degree of influence. Since the environment forces \(\alpha_k < 1\), the system never fully converges to \(T\) but reaches a steady-state blend between the evolving dynamics and the target.
\begin{figure}[ht]
  \begin{center}
\centerline{\includegraphics[width=0.85\columnwidth]{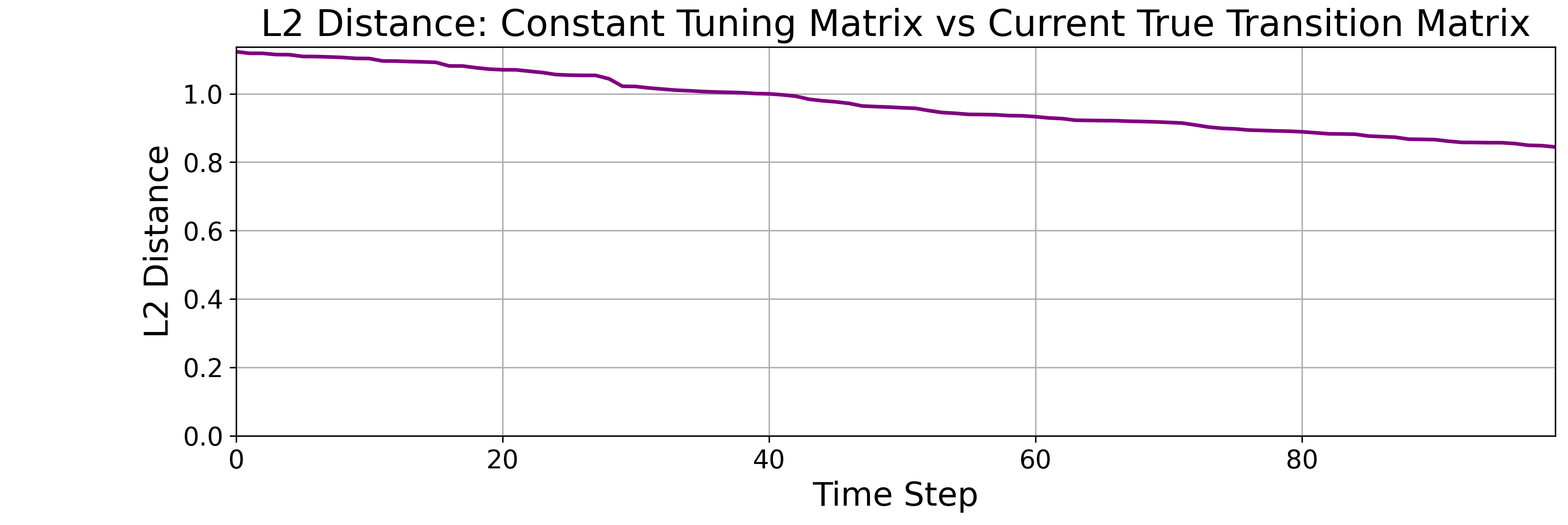}}
    \caption{$\ell_2$ distance between tuning matrix}
    \label{figure: L2 distance evolution}
  \end{center}
\vspace{-0.2in}
\end{figure}
Figure \ref{figure: distance to desired state}  demonstrates convergence of the $\ell_2$ distance between the current and desired steady-state distributions, decreasing from approximately $0.4$ at $k=0$ decaying to almost $0.1$. Together with the learned category shown in Figure \ref{figure: category belief evolution}, these results indicate robustness to category uncertainty where the agent achieves bounded control error and stable performance. This robustness is reflected not only in the steady-state tracking behavior but also in the overall measurement--cost trade-off.
\begin{figure}[ht]
  \begin{center}
\centerline{\includegraphics[width=0.85\columnwidth]{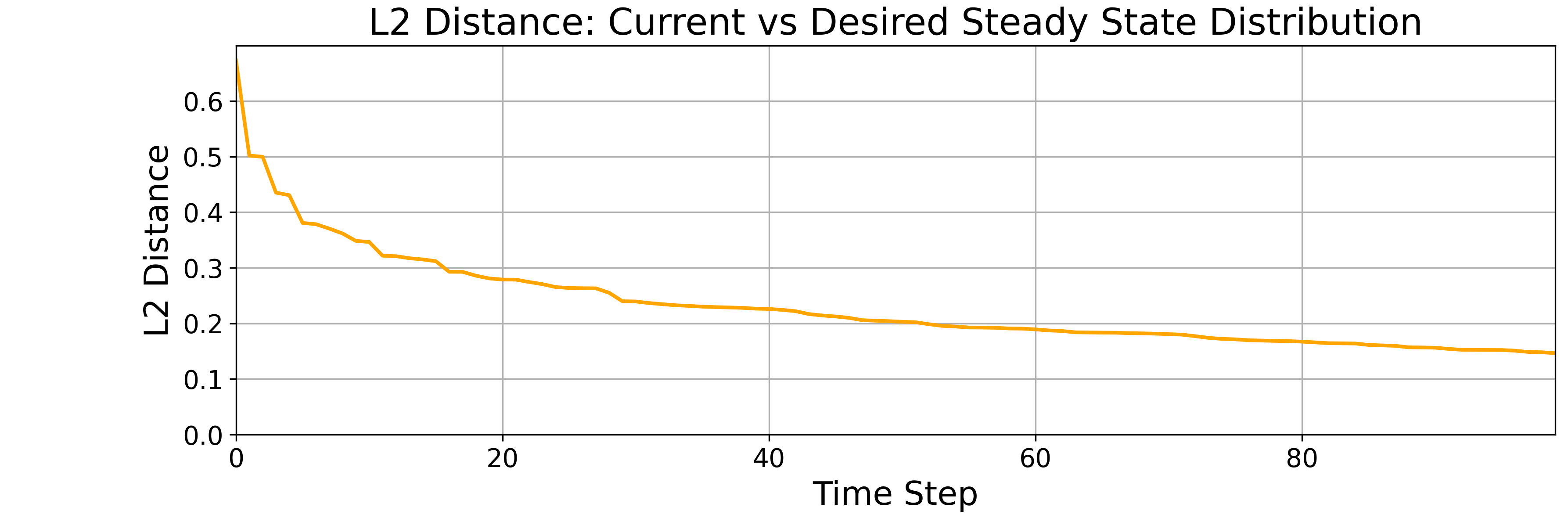}}
    \caption{$\ell_2$ distance between current and desired steady-state distributions}
    \label{figure: distance to desired state}
  \end{center}
\vspace{-0.25in}
\end{figure}
We compared five measurement strategies over 100 simulated patient trajectories (100 steps each) to quantify the trade-off between information acquisition and control performance. Figure \ref{figure: total measurement} summarizes total observations and total cost across agents which highlights the clear differences in measurement efficiency and overall objective value.
\begin{figure}[ht]
  \begin{center}
\centerline{\includegraphics[width=0.85\columnwidth]{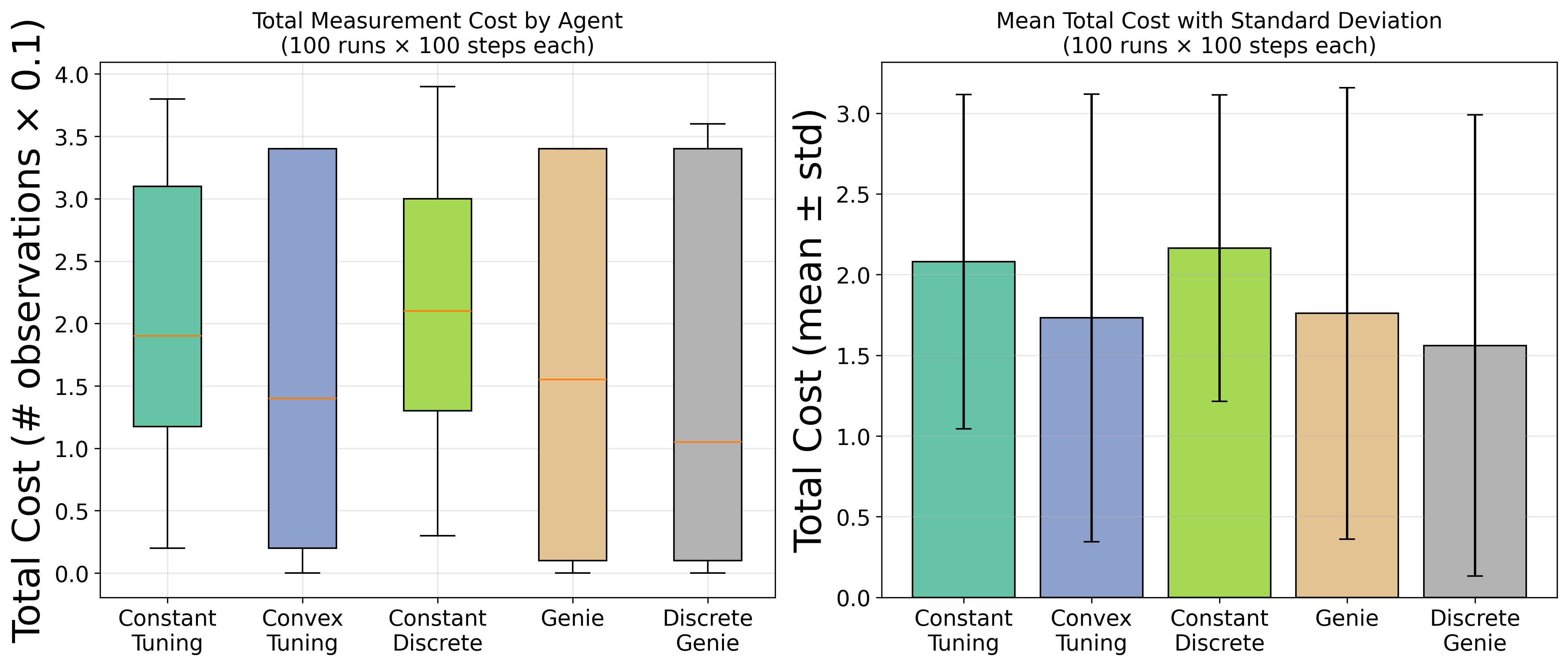}}
    \caption{Total measurement and total cost statistics for five measurement policies evaluated over 100 independent patient trajectories of 100 time steps each.}
    \label{figure: total measurement}
  \end{center}
\vspace{-0.25in}
\end{figure}
The  Discrete Genie achieves the lowest median cost ($1.0$ observation) and a low mean cost ($1.6 \pm 1.4$), consistent with the advantage of oracle category knowledge for selectively triggering measurements. {Convex Tuning attains the best mean performance ($ 1.75 \pm 1.4$), suggesting that belief-weighted tuning can approach oracle efficiency without perfect identification. In contrast, Constant Discrete incurs the highest average cost ($2.2 \pm 1.0$), indicating that fixed schedules are less measurement-efficient than entropy-driven policies.

To illustrate the treatment effect on expected patient survival, Figure~\ref{figure: life expectancy} shows the instantaneous expected life expectancy $v_i$ (Lemma~\ref{lemma: life expectancy}) of the evolving transition matrix $P_k$ for an advanced-stage young patient receiving EGFR-targeted therapy. As treatment blends into the baseline dynamics via $P_{k+1} = (1-\alpha_k)P_k + \alpha_k T$, the expected remaining life from state~A (Attenuation) increases from $237$ to $327$ time steps, with similar improvements from states B and C. The concave saturation shape reflects the exponentially decaying tuning authority $\alpha_k$, concentrating the majority of therapeutic benefit in the early treatment window.
\begin{figure}[ht]
  \begin{center}
\centerline{\includegraphics[width=0.85\columnwidth]{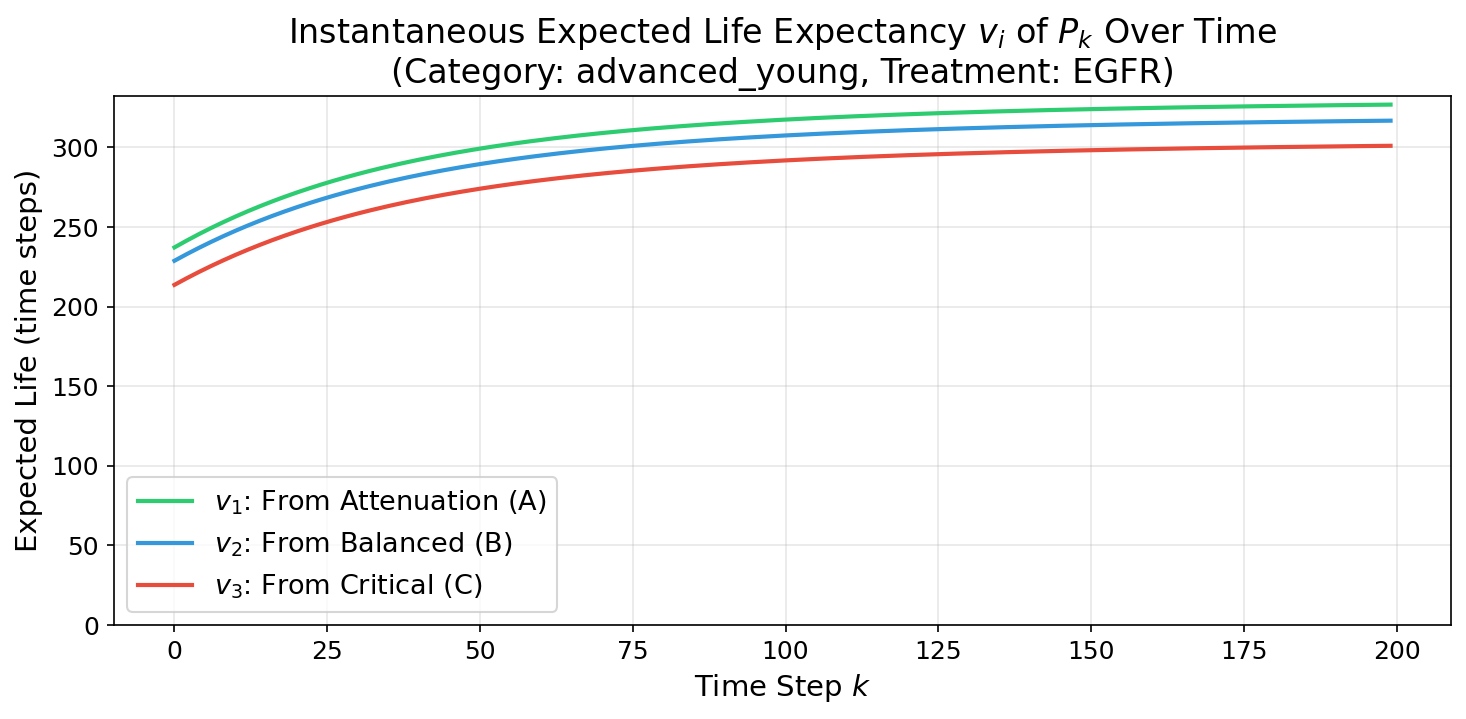}}
    \caption{Instantaneous life expectancy $v_i$ from each transient state as the transition matrix $P_k$ evolves under EGFR treatment for the advanced-stage young patient category. Life expectancy is computed at each step by solving $(I-P_k)v=r$ (Lemma~\ref{lemma: life expectancy}).}
    \label{figure: life expectancy}
  \end{center}
  \vspace{-0.3in}
\end{figure}

\section{Conclusion}


We presented a belief-space control framework for personalized cancer treatment under partial observability and explicit measurement constraints, grounded in the free-energy principle and active inference. Our approach targeted the evolution of patient state distributions rather than pointwise trajectories, capturing the plastic nature of oncologic interventions and latent patient heterogeneity. Empirical results using real clinical data demonstrate that the proposed framework simultaneously drives the belief over patient state toward a desired target distribution and maintains informative estimates of latent patient categories, all within a realistic measurement budget and over a clinically meaningful treatment horizon. 

In this work, the agent optimizes a fixed control objective given its assumed generative model, but it does not learn to improve its own decision-making strategy over time. While inference is adaptive through belief updates, control is fixed and is based on models learned offline. A promising direction for future work is the integration of online learning–based action selection within the active inference framework, enabling adaptive optimization of both treatment and measurement policies directly from data.

\bibliographystyle{IEEEtran}
\bibliography{bibliography}

\appendices
\section{Notation}

Table \ref{table: notation} summarizes the mathematical notation used throughout this paper.

\begin{table*}[!t]
\caption{Notation for variables}
\label{table: notation}
\centering
\footnotesize
\begin{tabular}{@{}ll ll ll@{}}
\toprule
\textbf{Symbol} & \textbf{Definition} &
\textbf{Symbol} & \textbf{Definition} &
\textbf{Symbol} & \textbf{Definition} \\
\midrule
$A$ & action pair of $(M,T)$
& $X$ & state
& $\alpha$ & tuning limit \\

$C$ & category
& $Y$ & observation
& $\gamma$ & discount factor \\

$H$ & entropy
& $d$ & number of categories
& $\theta$ & tuning policy \\

$J$ & cost function
& $e$ & standard basis vector
& $\kappa$ & control cost \\

JSD & Jensen-Shannon div.
& eig & set of eigenvalues
& $\lambda$ & eigenvalue \\

KL & Kullback-Leibler div.
& $i,j$ & dummy variables
& $\mu$ & measurement policy \\

$M$ & measurement decision
& $k$ & time
& $\pi_{X,C}$ & category and state belief \\

$P$ & transition matrix
& $n$ & number of states
& $\pi_X^*$ & desired/steady distribution \\

$T$ & tuning/treatment matrix
& $p$ & probability mass function
& $\tau$ & tracking cost \\
\bottomrule
\end{tabular}
\end{table*}

\section{Proof of Lemma \ref{lemma: life expectancy}} \label{section: proof of life expectancy lemma}
Let $I$ be the $n\times n$ identity matrix. The homogeneous portion, $(I-P)\vec{v}=0$ of the equation has infinite solutions in the form $\beta \vec{e}$ for all real $beta$, where $\vec{e}=[1\ 1\ \ldots \ 1]^{\mathsf T}$. So, if $\vec{v}^*$ solves the equation, then all $\vec{v}=\vec{v}^*+\beta \vec{e}$ also solves it. A solution to the equation exists if $\vec{r}$ is in the column space of $I-P$. We also know that the steady-state distribution $\pi_X$ is in the left null space of $I-P$. Therefore, a solution exists if and only if $\pi_X \vec{r}=0$, which is true, since $\pi_X$ has all $0$ entries other than the terminal state and $\vec{r}$ has a $0$ entry for the terminal state.

Now, let $\vec{v}^*$ be the coordinates of $\vec{r}$ in the column space of $I-P$. Then, the solution to our core equation has the form $\vec{v}^* + \beta \vec{e}$. The unique solution can be found using the additional condition $\pi_X \vec{v}=0$.

\section{Proof of Lemma \ref{lemma: eigenvalue of convex combination}} \label{section: proof of lemma}
\subsection{Eigenvalues in Convex Hull}
For the new transition matrix \(P_{k+1|C}\), \(\text{eig}(P_{k+1|C})\subset\conv\adjcurl{\eig P_{k|C}\cup\eig T_k}\) where \(\conv\) denotes convex hull.
\begin{proof}
Let \(S=\conv\adjcurl{\eig P_{k|C}\cup\eig T_k}\). For contradiction, assume that there exists an eigenvalue \(\lambda\) of \(P_{k+1|C}\) such that \(\lambda\notin S\). Since \(\lambda\) lies outside the convex hull and this convex hull is a closed convex set in \(\mathbb{C}\), by the Separating Hyperplane Theorem, there exists a hyperplane that strictly separates \(\lambda\) from \(S\). More precisely, there exists a non-zero linear functional \(\phi:\mathbb{C}\to\mathbb{R}\) and a constant \(c\in\mathbb{R}\) such that \(\phi(\lambda)>c\) and \(\phi(S)<c\). We can write \(\phi(z)=\text{Re}(w\cdot z)\) for some non-zero \(w\in\mathbb{C}\), where \(\text{Re}\) denotes the real part. Since \(S\) contains all eigenvalues of both \(P_{k|C}\) and \(T_k\), \(\text{Re}(w\cdot\lambda')<c\) for all eigenvalues \(\lambda'\) of \(P_{k|C}\) and \(\text{Re}(w\cdot\lambda'')<c\) for all eigenvalues \(\lambda''\) of \(T_k\). Since \(\lambda\) is an eigenvalue of \(P_{k+1|C}\), there exists a non-zero eigenvector \(v\) such that \(P_{k+1|C}v = \lambda v\). Thus, \((1-\alpha_{k|C})P_{k|C}v + \alpha_{k|C}T_kv = \lambda v\).

Taking the inner product with \(v^*\), normalizing \(vv^*=1\) and applying the functional \(\phi\) to both sides we get \(\text{Re}(\lambda w) = (1-\alpha_{k|C})\text{Re}(w\cdot v^*P_{k|C}v) + \alpha_{k|C}\text{Re}(w\cdot v^*T_kv)\). By the Toeplitz-Hausdorff Theorem, the numerical range of any matrix is a convex set containing all eigenvalues of that matrix. Thus, \(v^*P_{k|C}v\in\conv(\eig P_{k|C})\) and \(v^*T_kv\in\conv(\eig T_k)\). Since both lie in \(S\), we have \(\text{Re}(w\cdot v^*P_{k|C}v)<c, \text{Re}(w\cdot v^*T_kv)<c\) and \(\phi(\lambda)=\text{Re}(\lambda w) = (1-\alpha_{k|C})\text{Re}(w\cdot v^*P_{k|C}v) + \alpha_{k|C}\text{Re}(w\cdot v^*T_kv)< c\). This contradicts our assumption from that \(\phi(\lambda)>c\).
\end{proof}
\subsection{Eigenvalue Perturbation Bound}
For any \(\alpha_{k|C}\) and \(\lambda\in\eig P_{k+1|C}\), there exists a \(\lambda'\in\eig P_{k|C}\) such that, \(\abs{\lambda}\leq\abs{\lambda'}+\alpha_{k|C}2\sqrt{2}\).

\begin{proof}
Rewrite \(P_{k+1|C}\) as: \(P_{k+1|C}=P_{k|C}+\alpha_{k|C}(T_k-P_{k|C})=P_{k|C}+E\) where \(E =\alpha_{k|C}(T_k-P_{k|C})\) is the perturbation matrix.

By the Bauer-Fike theorem, if \(P_{k|C}\) is diagonalizable with \(P_{k|C}=XDX^{-1}\) where \(D\) contains the eigenvalues of \(P_{k|C}\), then for any eigenvalue \(\lambda\) of the perturbed matrix \(P_{k+1|C}\), there exists an eigenvalue \(\lambda'\) of \(P_{k|C}\) such that \(\abs{\lambda - \lambda'}\leq\kappa(X)\norm{E}_1\) where \(\kappa(X) = \norm{X}_1\norm{X^{-1}}_1\) is the condition number of the eigenvector matrix.

For stochastic matrices, we have \(\norm{E}_1=\norm{\alpha_{k|C}(T_k-P_{k|C})}_1=\abs{\alpha_{k|C}}\norm{T_k - P_{k|C}}_1\). Since both \(T_k\) and \(P_{k|C}\) are stochastic matrices (with entries in \([0,1]\) and rows summing to \(1\), we have \(\norm{T_k - P_{k|C}}_1\leq 2\). For the condition number, stochastic matrices typically have \(\kappa(X) \leq \sqrt{2}\) under appropriate normalization. Therefore \(\abs{\lambda}\leq\abs{\lambda'}+\alpha_{k|C}2\sqrt{2}\).
\end{proof}

\subsection{Second Largest Absolute Value Eigenvalue Upper Bound}
For the second largest absolute value eigenvalues \(\lambda_2,\lambda_2',\lambda_2''\) of \(P_{k+1|C},P_{k|C},T_k\), we have \(\abs{\lambda_2}\leq \max\{\abs{\lambda_2'},\abs{\lambda_2''}\}\).

\begin{proof}
This a consequence of the previous results on convex hull of eigenvalues and perturbation bound.
\end{proof}

\subsection{Exponential Convergence of Distributions and Divergences}
For any initial distribution \(\pi_{0,X}\) and distribution \(g\), \(H(\pi_{0,X}P^k)\to H(v_1)\) and \(JSD(\pi_{0,X}P^k\|g)\to JSD(v_1\|g)\) exponentially.
\begin{proof}
We only show the result for entropy. The result for JSD is similar. Let \(X_{ik}\) denote the state of the Markov chain at time \(k\) where \(X_0=i\). Then, after \(k\) steps, \(\pi_{k,X}=e_iP^k\) and \(H(\pi_{k,X})=-\sum_{j=1}^np_{ijk}\log p_{ijk}\) where \(p_{ijk}\) is the probability of transition from state \(i\) to \(j\) in \(k\) steps.
\begingroup
\setlength{\abovedisplayskip}{3pt}
\setlength{\belowdisplayskip}{3pt}
\setlength{\abovedisplayshortskip}{3pt}
\setlength{\belowdisplayshortskip}{3pt}
\setlength{\jot}{2pt}

\begin{align}
\abs{H(X_{ik})-H(v_1)}
&= \abs{
-\sum_{j=1}^n p_{ijk}\log p_{ijk}
-\adjpar{-\sum_{j=1}^n v_{1j}\log v_{1j}}
} \\
&= \abs{
\sum_{j=1}^n -p_{ijk}\log p_{ijk}
+ v_{1j}\log v_{1j}
}.
\end{align}

Without loss of generality assume \(v_{1j}>0\) for all \(j\leq m\)
and \(v_{1j}=0\) for all \(j>m\). Then,

\begin{equation}
\begin{aligned}
\abs{H(X_{ik})-H(v_1)}
&= \Bigg|
\sum_{j>m}-p_{ijk}\log p_{ijk} \\
&\quad
+\sum_{j=1}^{m}-p_{ijk}\log p_{ijk}
+v_{1j}\log v_{1j}
\Bigg| .
\end{aligned}
\end{equation}

\begin{equation}
\begin{aligned}
&= \Bigg|
\sum_{j>m}-p_{ijk}\log p_{ijk} \\
&\quad
+\sum_{j=1}^{m}
-(p_{ijk}-v_{1j}+v_{1j})\log p_{ijk}
+v_{1j}\log v_{1j}
\Bigg| .
\end{aligned}
\end{equation}

\begin{equation}
\begin{aligned}
&= \Bigg|
\sum_{j>m}-p_{ijk}\log p_{ijk} \\
&\quad
-\sum_{j=1}^{m}
(p_{ijk}-v_{1j}+v_{1j})\log p_{ijk}
-v_{1j}\log v_{1j}
\Bigg| .
\end{aligned}
\end{equation}

\begin{equation}
\begin{aligned}
&= \Bigg|
\sum_{j>m}-p_{ijk}\log p_{ijk} \\
&\quad
-\sum_{j=1}^{m}
(p_{ijk}-v_{1j})\log p_{ijk}
+v_{1j}(\log p_{ijk}-\log v_{1j})
\Bigg| .
\end{aligned}
\end{equation}

\begin{equation}
\begin{aligned}
\abs{H(X_{ik})-H(v_1)}
&\leq
\sum_{j>m}
\abs{-p_{ijk}\log p_{ijk}} \\
&\quad
+\sum_{j=1}^{m}
\abs{-(p_{ijk}-v_{1j})\log p_{ijk}} \\
&\quad
+\sum_{j=1}^{m}
\abs{v_{1j}(\log p_{ijk}-\log v_{1j})} .
\end{aligned}
\end{equation}

\begin{equation}
\begin{aligned}
&\leq
\sum_{j>m}
-c_1c_2^k\log p_{ijk} \\
&\quad
+\sum_{j=1}^{m}
-c_1c_2^k\log p_{ijk} \\
&\quad
+\sum_{j=1}^{m}
\abs{v_{1j}(\log p_{ijk}-\log v_{1j})} .
\end{aligned}
\end{equation}

\begin{equation}
\begin{aligned}
&\leq
\sum_{j>m}
-c_1c_2^k\log p_{ijk} \\
&\quad
+\sum_{j=1}^{m}
-c_1c_2^k\log p_{ijk} \\
&\quad
+\sum_{j=1}^{m}
\abs{v_{1j}(\log p_{ijk}-\log v_{1j})} .
\end{aligned}
\end{equation}

\begin{equation}
\begin{aligned}
&\leq
\sum_{j>m}
c_1c_2^k c_3^k
+\sum_{j=1}^{m}
c_1c_2^k c_3^k \\
&\quad
+\sum_{j=1}^{m}
v_{1j}
\frac{c_1c_2^k}{\min\{v_{1j},p_{ijk}\}} .
\end{aligned}
\end{equation}

\endgroup
where we used the pointwise convergence of the distribution (\(p_{ijk}\to v_{1j}\)) and the fact that \(p_{ijk}\geq \frac{1}{c_3^k}\). The latter holds because \(P\) is an aperiodic stochastic matrix with at most one absorbing state and no cycles among transient states; by Chapter~3 of \cite{kemeny1976finite}, all eigenvalues of \(P\) other than \(\lambda_1=1\) satisfy \(\abs{\lambda_i}<1\), which guarantees exponential convergence of \(p_{ijk}\to v_{1j}\) and thus the stated lower bound for sufficiently large \(k\).
\end{proof}

\section{Dataset Details}
\label{app:dataset}

\subsection{Patient Timeline}
Figure~\ref{fig:patient_timeline} illustrates the longitudinal signals available for a representative patient.

\begin{figure}[!ht]
  \centering
  \includegraphics[width=\columnwidth]{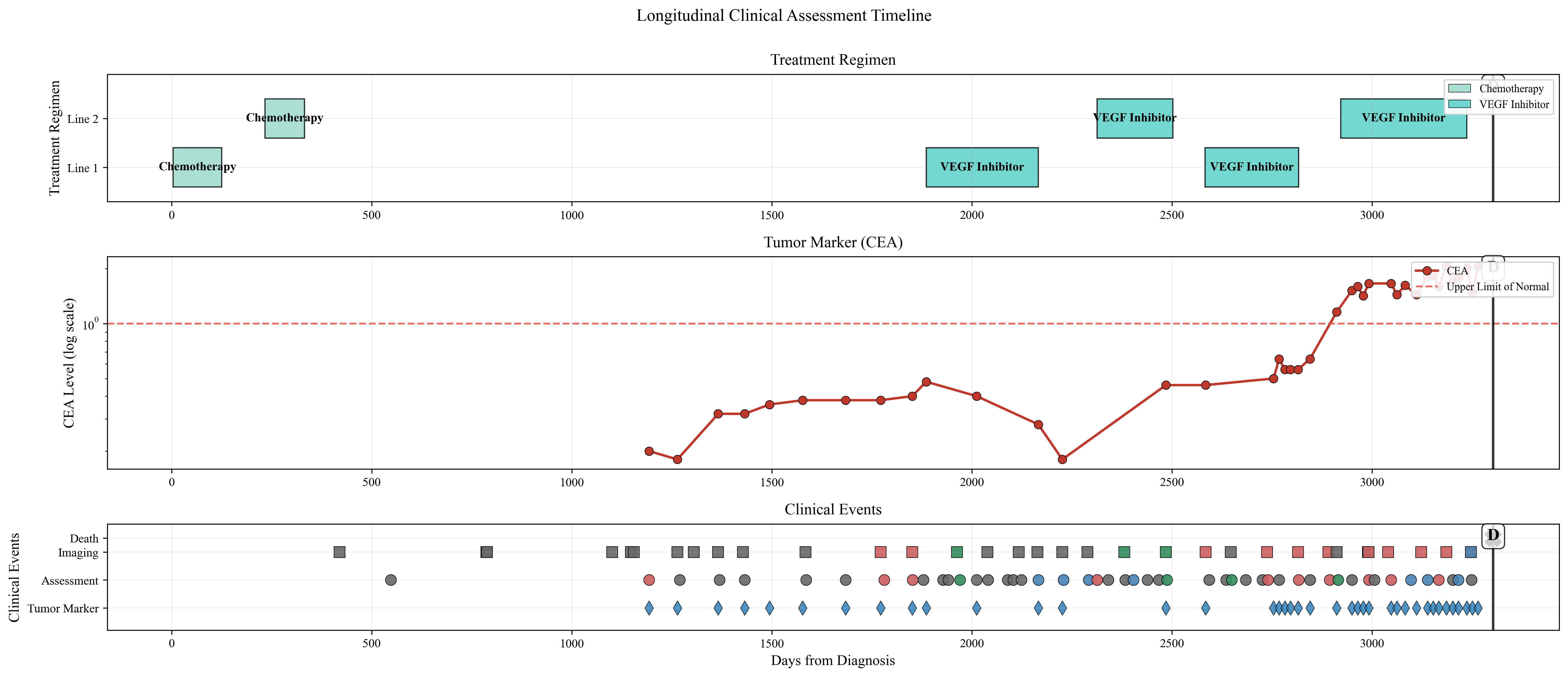}
  \caption{Available longitudinal information for a representative patient in the AACR Project GENIE BPC CRC cohort, including treatment history, biomarker measurements, radiological/imaging assessments, and oncologist evaluations over time (days from diagnosis).}
  \label{fig:patient_timeline}
\end{figure}

\subsection{Dataset Summary and Preprocessing Outputs}
\label{app:dataset_summary}

Table~\ref{tab:genie_crc_processing} summarizes the derived quantities and preprocessing outputs used for state/action modeling.

\noindent \begin{table*}[!ht]
\centering
\small
\caption{GENIE BPC CRC v2.0-public: non-genomic trajectory construction summary and data support for state/action modeling.}
\label{tab:genie_crc_processing}
\begin{tabular}{p{0.16\linewidth} p{0.58\linewidth} p{0.18\linewidth}}
\toprule
\textbf{Component} & \textbf{Operationalization (non-genomic)} & \textbf{Summary} \\
\midrule
State space &
Disease states $\{A,B,C,D\}$ inferred from imaging and oncologist assessments; death events define absorbing $D$. &
See Fig.~\ref{fig:transitions} \\
\addlinespace
State evidence &
Imaging and oncologist assessments provide categorical status: improving/responding, stable/no change, progressing/worsening, not stated/intermediate; oncology assessments are curated approximately monthly. &
--- \\
\addlinespace
Patient stratification &
Age groups: young ($<50$), middle (50--70), elderly ($>70$); Stage groups: early (I--III) vs. advanced (IV). &
Age: 515 / 802 / 168;\newline
Stage: 783 / 701 \\
\addlinespace
Baseline dynamics $P_0$ &
Estimated from treatment-free intervals (exclude observation events during active treatment); death transitions can use all intervals. &
$\sim$90K untreated events;\newline
$\sim$47K transitions \\
\addlinespace
Treatment-conditioned dynamics &
Transitions extracted during treated windows after mapping regimens to action classes. &
See Appendix~\ref{app:transition_counting} \\
\midrule
\textbf{Action class} & \multicolumn{2}{l}{\textbf{Examples / definition}} \\
\midrule
ChemoOnly & \multicolumn{2}{l}{Cytotoxic chemotherapy without targeted/IO agents (e.g., 5-FU/capecitabine, oxaliplatin, irinotecan).} \\
VEGF & \multicolumn{2}{l}{Anti-angiogenic agents targeting VEGF pathway (bevacizumab, aflibercept, ramucirumab).} \\
EGFR & \multicolumn{2}{l}{Anti-EGFR monoclonal antibodies (cetuximab, panitumumab).} \\
IO & \multicolumn{2}{l}{Immune checkpoint inhibitors (pembrolizumab, nivolumab, ipilimumab).} \\
HER2 & \multicolumn{2}{l}{HER2-targeted therapies (trastuzumab, pertuzimab, tucatinib).} \\
BRAF & \multicolumn{2}{l}{BRAF inhibitors (encorafenib, vemurafenib, dabrafenib).} \\
Investigational & \multicolumn{2}{l}{Agents explicitly marked investigational in regimen strings.} \\
\bottomrule
\end{tabular}
\end{table*}

\subsection{State Mapping}
\label{app:state_mapping}
To discretize the patient status, we map raw clinical assessments from imaging reports, medical oncologist assessments (MOA), and vital status registries to the state space $\mathcal{X} = \{A, B, C, D\}$. The specific rules used to convert these heterogeneous inputs into unified state labels are detailed in Table~\ref{tab:state_mapping}.
\begin{table}[!ht]
\centering
\small
\caption{Mapping of clinical assessments to discrete disease states.}
\label{tab:state_mapping}
\begin{tabular}{lll}
\toprule
\textbf{Source} & \textbf{Raw Assessment} & \textbf{State} \\
\midrule
Imaging & Complete/Partial Response & $A$ (Attenuation of tumor) \\
Imaging & Stable Disease & $B$ (Balanced response) \\
Imaging & Progressive Disease & $C$ (Critical condition) \\
\addlinespace
MOA & Improving/Responding & $A$ (Attenuation of tumor) \\
MOA & Stable/No Change & $B$ (Balanced response) \\
MOA & Progressing/Worsening & $C$ (Critical condition) \\
MOA & Not Evaluated/Indeterminate & Unassigned \\
\addlinespace
Vital Status & Deceased & $D$ (Death) \\
\bottomrule
\end{tabular}
\end{table}

\subsection{Transition Counting and Probability Estimation}
\label{app:transition_counting}

We construct a discrete-time trajectory for each patient by ordering clinical assessments by timestamp (days from diagnosis) and assigning a discrete disease state $X_k \in \mathcal{X}$ at each observation index $k=0,1,\dots$. Let $C \in \mathcal{C}$ denote the patient category (e.g., stage $\times$ age group). For each transition step, we record the \emph{treatment class} $t_k \in \mathcal{T}$ that is active at the time of the \emph{starting} observation, and a measurement selector $M_{k+1}\in\{0,1\}$ indicating whether an informative measurement (e.g., imaging) is obtained at the next stage. We use the composite action notation $A_k \triangleq (M_{k+1},t_k)$.

We count transitions only between \emph{validly assigned} states: indices $k$ are included only if both endpoints satisfy $X_k \in \mathcal{X}$ and $X_{k+1}\in\mathcal{X}$ (observations labeled not stated/indeterminate are not used as transition endpoints).

\paragraph{Baseline dynamics $P_{0\mid C}$.}
For each category $c \in \mathcal{C}$, we estimate the baseline transition matrix $P_{0\mid c}$ by counting transitions between \emph{consecutive observations} $(X_k \rightarrow X_{k+1})$ whose starting observation occurs in a treatment-free interval (i.e., no active regimen at the time of the starting observation). Let $\mathcal{K}^{0}_c$ denote the set of such indices. Define empirical counts
\begin{equation}
\begin{aligned}
{P}_{0\mid c}(i,j)
&= \Pr\!\left(
X_{k+1}=j
\,\middle|\,
X_k=i,\, C=c,\, t_k=\varnothing
\right) \\
&= \frac{N^{0}_{ij\mid c}}
{\sum_{j'\in\mathcal{X}} N^{0}_{ij'\mid c}} .
\end{aligned}
\label{eq:P0_est}
\end{equation}

\paragraph{Category-conditioned treatment matrices $T_{t\mid c}$.}
To obtain empirical treatment transition matrices consistent with \eqref{equation: transition matrix update}, we estimate one transition matrix for each treatment class $t \in \mathcal{T}$ \emph{within each category} $c\in\mathcal{C}$ by counting consecutive-observation transitions whose starting observation occurs while treatment class $t_k=t$ is active. Let $\mathcal{K}^{t}_c$ denote the set of indices satisfying $C=c$ and $t_k=t$ at the starting observation. Define
\begin{equation}
\begin{aligned}
{T}_{t\mid c}(i,j)
&= \Pr\!\left(
X_{k+1}=j
\,\middle|\,
X_k=i,\, C=c,\, t_k=t
\right) \\
&= \frac{N^{t}_{ij\mid c}}
{\sum_{j'\in\mathcal{X}} N^{t}_{ij'\mid c}} .
\end{aligned}
\label{eq:T_action_est}
\end{equation}
When simulating \eqref{equation: transition matrix update}, we instantiate the matrix $T_k$ by selecting the category-conditioned matrix corresponding to the chosen treatment class: if the realized category is $C=c$ and the selected class at stage $k$ is $t_k=t$, we set $T_k \equiv T_{t\mid c}$.

\end{document}